\documentclass{article}
\usepackage{algorithm}
\usepackage{algorithmicx}
\usepackage[noend]{algpseudocode}
\usepackage{amsmath}
\usepackage{amssymb}
\usepackage{amsthm}
\usepackage{bbm}
\usepackage{bm}
\usepackage{color}
\usepackage{dirtytalk}
\usepackage{dsfont}
\usepackage{enumerate}
\usepackage{graphicx}
\usepackage[accepted]{icml2019}
\usepackage{mathtools}
\usepackage{subfigure}
\usepackage{times}

\usepackage[usenames,dvipsnames]{xcolor}
\usepackage[bookmarks=false]{hyperref}
\hypersetup{
  pdftex,
  pdffitwindow=true,
  pdfstartview={FitH},
  pdfnewwindow=true,
  colorlinks,
  linktocpage=true,
  linkcolor=Green,
  urlcolor=Green,
  citecolor=Green
}
\usepackage[capitalize,noabbrev]{cleveref}

\newcommand{\commentout}[1]{}
\newcommand{\junk}[1]{}

\newtheorem{theorem}{Theorem}

\newtheorem{lemma}{Lemma}

\newcommand{\cB}{\mathcal{B}}
\newcommand{\cF}{\mathcal{F}}
\newcommand{\cG}{\mathcal{G}}
\newcommand{\cH}{\mathcal{H}}
\newcommand{\eps}{\varepsilon}

\newcommand{\realset}{\mathbb{R}}

\newcommand{\E}[1]{\mathbb{E} \left[#1\right]}
\newcommand{\condE}[2]{\mathbb{E} \left[#1 \,\middle|\, #2\right]}
\newcommand{\prob}[1]{\mathbb{P} \left(#1\right)}
\newcommand{\condprob}[2]{\mathbb{P} \left(#1 \,\middle|\, #2\right)}
\newcommand{\var}[1]{\mathrm{var} \left[#1\right]}
\newcommand{\condvar}[2]{\mathrm{var} \left[#1 \,\middle|\, #2\right]}

\newcommand{\abs}[1]{\left|#1\right|}
\newcommand{\ceils}[1]{\left\lceil#1\right\rceil}
\newcommand*\dif{\mathop{}\!\mathrm{d}}
\newcommand{\floors}[1]{\left\lfloor#1\right\rfloor}
\newcommand{\I}[1]{\mathds{1} \! \left\{#1\right\}}

\newcommand{\set}[1]{\left\{#1\right\}}

\newcommand{\transpose}{^\mathsf{\scriptscriptstyle T}}

\DeclareMathOperator*{\argmax}{arg\,max\,}

\mathchardef\mhyphen="2D

\newcommand{\bootstrapts}{{\tt BootstrapThompson}}

\newcommand{\egreedy}{{\tt EG}}

\newcommand{\giro}{{\tt Giro}}

\newcommand{\klucb}{{\tt KL\mhyphen UCB}}

\newcommand{\lints}{{\tt LinTS}}
\newcommand{\linucb}{{\tt LinUCB}}

\newcommand{\ts}{{\tt TS}}
\newcommand{\ucb}{{\tt UCB1}}
\newcommand{\ucbglm}{{\tt UCB\mhyphen GLM}}

\icmltitlerunning{Garbage In, Reward Out: Bootstrapping Exploration in Multi-Armed Bandits}

\begin{document}

\twocolumn[
\icmltitle{Garbage In, Reward Out: Bootstrapping Exploration in Multi-Armed Bandits}

\icmlsetsymbol{equal}{*}

\begin{icmlauthorlist}
\icmlauthor{Branislav Kveton}{gr}
\icmlauthor{Csaba Szepesv\'ari}{dp,ua}
\icmlauthor{Sharan Vaswani}{mila}
\icmlauthor{Zheng Wen}{ar}
\icmlauthor{Mohammad Ghavamzadeh}{fair}
\icmlauthor{Tor Lattimore}{dp}
\end{icmlauthorlist}

\icmlaffiliation{ar}{Adobe Research}
\icmlaffiliation{dp}{DeepMind}
\icmlaffiliation{fair}{Facebook AI Research}
\icmlaffiliation{gr}{Google Research}
\icmlaffiliation{mila}{Mila, University of Montreal}
\icmlaffiliation{ua}{University of Alberta}

\icmlcorrespondingauthor{Branislav Kveton}{bkveton@google.com}

\vskip 0.3in
]

\printAffiliationsAndNotice{}

\begin{abstract}
We propose a bandit algorithm that explores by randomizing its history of rewards. Specifically, it pulls the arm with the highest mean reward in a non-parametric bootstrap sample of its history with pseudo rewards. We design the pseudo rewards such that the bootstrap mean is optimistic with a sufficiently high probability. We call our algorithm $\giro$, which stands for \emph{garbage in, reward out}. We analyze $\giro$ in a Bernoulli bandit and derive a $O(K \Delta^{-1} \log n)$ bound on its $n$-round regret, where $\Delta$ is the difference in the expected rewards of the optimal and the best suboptimal arms, and $K$ is the number of arms. The main advantage of our exploration design is that it easily generalizes to structured problems. To show this, we propose contextual $\giro$ with an arbitrary reward generalization model. We evaluate $\giro$ and its contextual variant on multiple synthetic and real-world problems, and observe that it performs well.
\end{abstract}


\section{Introduction}
\label{sec:introduction}

A \emph{multi-armed bandit} \cite{lai85asymptotically,auer02finitetime,lattimore19bandit} is an online learning problem where actions of the \emph{learning agent} are represented by \emph{arms}. The arms can be treatments in a clinical trial or ads on a website. After the arm is \emph{pulled}, the agent receives its \emph{stochastic reward}. The objective of the agent is to maximize its expected cumulative reward. The agent does not know the expected rewards of the arms and thus faces the so-called \emph{exploration-exploitation dilemma}: \emph{explore}, and learn more about arms; or \emph{exploit}, and pull the arm with the highest estimated reward thus far.

A \emph{contextual bandit} \cite{li10contextual,agrawal13thompson} is a generalization of a multi-armed bandit where the learning agent has access to additional context in each round. The context can encode the medical data of a patient in a clinical trial or the demographic information of a targeted user on a website. In this case, the expected reward is an unknown function of the arm and context. This function is often parametric and its parameters are learned. In \emph{linear bandits} \cite{rusmevichientong10linearly,dani08stochastic,abbasi-yadkori11improved}, this function is linear; and the expected reward is the dot product of a known context vector and an unknown parameter vector.

Arguably, the most used and studied exploration strategies in multi-armed and contextual bandits are \emph{Thompson sampling} \cite{thompson33likelihood,agrawal13further}, the \emph{optimism in the face of uncertainty} \cite{auer02finitetime,abbasi-yadkori11improved}, and the \emph{$\epsilon$-greedy policy} \cite{sutton98reinforcement,auer02finitetime}. The \emph{$\epsilon$-greedy policy} is general and thus widely used in practice. However, it is also statistically suboptimal. Its performance heavily depends on the value of $\epsilon$ and the strategy for annealing it.

\emph{Optimism in the face of uncertainty (OFU)} relies on high-probability confidence sets. These sets are statistically and computationally efficient in multi-armed and linear bandits \cite{auer02finitetime,abbasi-yadkori11improved}. However, when the reward function is non-linear in context, we only know how to construct approximate confidence sets \cite{filippi10parametric,zhang16online,li17provably,jun17scalable}. These sets tend to be conservative \cite{filippi10parametric} and statistically suboptimal.

The key idea in \emph{Thompson sampling (TS)} is to maintain a posterior distribution over model parameters and then act optimistically with respect to samples from it. TS is computationally efficient when the posterior distribution has a closed form, as in multi-armed bandits with Bernoulli and Gaussian rewards. If the posterior does not have a closed form, it has to be approximated. Computationally efficient approximations exist in multi-armed bandits with $[0, 1]$ rewards \cite{agrawal13further}. Such approximations are costly in general \cite{gopalan14thompson,kawale15efficient,lu17ensemble,riquelme18deep}.

To address these problems, bootstrapping exploration has been proposed in both multi-armed and contextual bandits \cite{baransi14subsampling,eckles14thompson,osband15bootstrapped,tang15personalized,elmachtoub17practical,vaswani18new}. Bootstrapping has two advantages over existing exploration strategies. First, unlike OFU and TS, it is easy to implement in any problem, because it does not require problem-specific confidence sets or posteriors. Second, unlike the $\epsilon$-greedy policy, it is data driven and not sensitive to tuning. Despite its advantages and good empirical performance, exploration by bootstrapping is poorly understood theoretically. The strongest theoretical result is that of \citet{osband15bootstrapped}, who showed that a form of bootstrapping in a Bernoulli bandit is equivalent to Thompson sampling.

We make the following contributions in this paper. First, we propose a general randomized algorithm that explores conditioned on its history. We show that some instances of this algorithm are not sound. Second, we propose $\giro$, an algorithm that pulls the arm with the highest mean reward in a non-parametric bootstrap sample of its history with pseudo rewards. We design the pseudo rewards such that the bootstrap mean is optimistic with a high probability. Third, we analyze $\giro$ in a $K$-armed Bernoulli bandit and prove a $O(K \Delta^{-1} \log n)$ bound on its $n$-round regret, where $\Delta$ is the difference in the expected rewards of the optimal and the best suboptimal arms. Our analyses of the general randomized algorithm and $\giro$ provide novel insights on how randomness helps with exploration. Fourth, we propose contextual $\giro$. Finally, we empirically evaluate $\giro$ and its contextual variant on both synthetic and real-world problems, and observe good performance.


\section{Setting}
\label{sec:setting}

We adopt the following notation. The set $\set{1, \dots, n}$ is denoted by $[n]$. We define $\mathrm{Ber}(x; p) = p^x (1 - p)^{1 - x}$ and let $\mathrm{Ber}(p)$ be the corresponding Bernoulli distribution. In addition, we define $B(x; n, p) = \binom{n}{x} p^x (1 - p)^{n - x}$ and let $B(n, p)$ be the corresponding binomial distribution. For any event $E$, $\I{E} = 1$ if event $E$ occurs and $\I{E} = 0$ otherwise. We introduce multi-armed and contextual bandits below.

A \emph{multi-armed bandit} \cite{lai85asymptotically,auer02finitetime,lattimore19bandit} is an online learning problem where the learning agent pulls $K$ arms in $n$ rounds. In round $t \in [n]$, the agent pulls arm $I_t \in [K]$ and receives its reward. The reward of arm $i \in [K]$ in round $t$, $Y_{i, t}$, is drawn i.i.d.\ from a distribution of arm $i$, $P_i$, with mean $\mu_i$ and support $[0, 1]$. The means are unknown to the learning agent. The objective of the agent is to maximize its expected cumulative reward in $n$ rounds.

We assume that arm $1$ is optimal, that is $\mu_1 > \max_{i > 1} \mu_i$. Let $\Delta_i = \mu_1 - \mu_i$ be the \emph{gap} of arm $i$, the difference in the expected rewards of arms $1$ and $i$. Then maximization of the expected $n$-round reward can be viewed as minimizing the \emph{expected $n$-round regret}, which we define as
\begin{align}
  R(n)
  = \sum_{i = 2}^K \Delta_i \E{\sum_{t = 1}^n \I{I_t = i}}\,.
  \label{eq:regret}
\end{align}

A \emph{contextual bandit} \cite{li10contextual,agrawal13thompson} is a generalization of a multi-armed bandit where the learning agent observes context $x_t \in \realset^d$ at the beginning of each round $t$. The reward of arm $i$ in round $t$ is drawn i.i.d.\ from a distribution that depends on both arm $i$ and $x_t$. One example is a logistic reward model, where
\begin{align}
  Y_{i, t}
  \sim \mathrm{Ber}(1 / (1 + \exp[- x_t\transpose \theta_i]))
  \label{eq:generalized linear model}
\end{align}
and $\theta_i \in \realset^d$ is the parameter vector associated with arm $i$. The learning agent does not know $\theta_i$.

If the reward $Y_{i, t}$ was generated as in \eqref{eq:generalized linear model}, we could solve our contextual bandit problem as a generalized linear bandit \cite{filippi10parametric}. However, if $Y_{i, t}$ was generated by a more complex function of context $x_t$, such as a neural network, we would not know how to design a sound bandit algorithm. The difficulty is not in modeling uncertainty; it is in the lack of computationally efficient methods to do it. For instance, Thompson sampling can be analyzed in very general settings \cite{gopalan14thompson}. However, as discussed in \cref{sec:introduction}, accurate posterior approximations are computationally expensive.


\section{General Randomized Exploration}
\label{sec:gre}

In this section, we present a general randomized algorithm that explores conditioned on its history and context. Later, in \cref{sec:giro,sec:analysis}, we propose and analyze an instance of this algorithm with sublinear regret.

We use the following notation. The \emph{history} of arm $i$ after $s$ pulls is a vector $\cH_{i, s}$ of length $s$. The $j$-th entry of $\cH_{i, s}$ is a pair $(x_\ell, Y_{i, \ell})$, where $\ell$ is the index of the round where arm $i$ is pulled for the $j$-th time. We define $\cH_{i, 0} = ()$. The \emph{number of pulls} of arm $i$ in the first $t$ rounds is denoted by $T_{i, t}$ and defined as $T_{i, t} = \sum_{\ell = 1}^t \I{I_\ell = i}$.

\begin{algorithm}[t]
  \caption{General randomized exploration.}
  \label{alg:gre}
  \begin{algorithmic}[1]
    \State $\forall i \in [K]: T_{i, 0} \gets 0, \, \cH_{i, 0} \gets ()$
    \Comment{Initialization}
    \For{$t = 1, \dots, n$}
      \For{$i = 1, \dots, K$}
      \Comment{Estimate arm values}
        \State $s \gets T_{i, t - 1}$
        \State Draw $\hat{\mu}_{i, t} \sim p(\cH_{i, s}, x_t)$
        \label{alg:gre:estimate}
      \EndFor
      \State $I_t \gets \argmax_{i \in [K]} \hat{\mu}_{i, t}$
      \Comment{Pulled arm}
      \label{alg:gre:pulled arm}
      \State Pull arm $I_t$ and observe $Y_{i, t}$
      \Statex \vspace{-0.05in}
      \For{$i = 1, \dots, K$}
      \Comment{Update statistics}
        \If{$i = I_t$}
          \State $T_{i, t} \gets T_{i, t - 1} + 1$
          \State $\cH_{i, T_{i, t}} \gets
          \cH_{i, T_{i, t - 1}} \oplus ((x_t, Y_{i, t}))$
          \label{alg:gre:update}
        \Else
          \State $T_{i, t} \gets T_{i, t - 1}$
        \EndIf
      \EndFor
    \EndFor
  \end{algorithmic}
\end{algorithm}

Our meta algorithm is presented in \cref{alg:gre}. In round $t$, the algorithm draws the value of each arm $i$, $\hat{\mu}_{i, t}$, from distribution $p$ (line \ref{alg:gre:estimate}), which depends on the history of the arm $\cH_{i, s}$ and the context $x_t$ in round $t$. The arm with the highest value is pulled (line \ref{alg:gre:pulled arm}), and its history is extended by a pair of context $x_t$ and the reward of arm $I_t$ (line \ref{alg:gre:update}). We denote by $u \oplus v$ the concatenation of vectors $u$ and $v$.

We present a general contextual algorithm because we return to it in \cref{sec:contextual giro,sec:contextual experiments}. For now, we restrict our attention to multi-armed bandits. To simplify exposition, we omit context from the entries of $\cH_{i, s}$.

\cref{alg:gre} is instantiated by designing the distribution $p$ in line \ref{alg:gre:estimate}. This distribution can be designed in many ways. In Bernoulli TS \cite{agrawal13further}, for instance, $p$ is a beta posterior distribution. A less direct approach to designing $p$ is to resample the history of rewards in each round, as in bootstrapping exploration (\cref{sec:related work}). More specifically, let $\cB_{i, s}$ be a \emph{non-parametric bootstrap sample} \cite{efron86bootstrap} of arm $i$ after $s$ pulls. We let $\cB_{i, s}$ be a vector of the same length as $\cH_{i, s}$ and assume that its entries are drawn with replacement from the entries of $\cH_{i, s}$ in each round. Then the value of arm $i$ in round $t$ is estimated as
\begin{align}
  \hat{\mu}_{i, t}
  = \frac{1}{\abs{\cB_{i, s}}} \sum_{y \in \mathcal{B}_{i, s}} y\,,
  \label{eq:follow the leader}
\end{align}
where $s = T_{i, t - 1}$. Note that we slightly abuse our notation and treat vectors as sets. In the next section, we show that this natural instance of \cref{alg:gre} can have linear regret.

\subsection{Linear Regret}
\label{sec:linear regret}

The variant of \cref{alg:gre} in \eqref{eq:follow the leader} can have linear regret in a Bernoulli bandit with $K = 2$ arms. More specifically, if the first reward of the optimal arm is $0$, its estimated value remains $0$ until the arm is pulled again, which may never happen if the estimated value of the other arm is positive. We formally state this result below.

Without loss of generality, let $\mu_1 > \mu_2$. \cref{alg:gre} is implemented as follows. Both arms are initially pulled once, arm $1$ in round $1$ and arm $2$ in round $2$. The value of arm $i$ in round $t$ is computed as in \eqref{eq:follow the leader}. If $\hat{\mu}_{1, t} = \hat{\mu}_{2, t}$, the tie is broken by a fixed rule that is chosen randomly in advance, as is common in multi-armed bandits. In particular, \cref{alg:gre} draws $Z \sim \mathrm{Ber}(1 / 2)$ before the start of round $1$. If $\hat{\mu}_{1, t} = \hat{\mu}_{2, t}$, $I_t = \I{Z = 1} + 1$. We bound the regret of this algorithm below.

\begin{lemma}
\label{lem:lower bound} In a Bernoulli bandit with $2$ arms, the expected $n$-round regret of the above variant of \cref{alg:gre} can be bounded from below as $R(n) \geq 0.5 \, (1 - \mu_1) \Delta_2 (n - 1)$.
\end{lemma}
\begin{proof}
By design, the algorithm does not pull arm $1$ after event $E = \set{Z = 1, \, \cH_{1, 1} = (0)}$ occurs. Since the events $Z = 1$ and $\cH_{1, 1} = (0)$ are independent,
\begin{align*}
  \prob{E}
  = \prob{Z = 1} \prob{\cH_{1, 1} = (0)}
  = 0.5 \, (1 - \mu_1)\,.
\end{align*}
Moreover, if event $\cH_{1, 1} = (0)$ occurs, it must occur by the end of round $1$ because the algorithm pulls arm $1$ in round $1$. Now we combine the above two facts and get
\begin{align*}
  R(n)
  & \geq \E{\left(\sum_{t = 2}^n \Delta_2 \I{I_t = 2}\right) \I{E}} \\
  & = \condE{\sum_{t = 2}^n \Delta_2 \I{I_t = 2}}{E} \prob{E} \\
  & = 0.5 \, (1 - \mu_1) \Delta_2 (n - 1)\,.
\end{align*}
This concludes the proof.
\end{proof}

A similar lower bound, $R(n) \geq 0.5 \, (1 - \mu_1)^k \Delta_2 (n - k)$, can be derived in the setting where each arm is initialized by $k$ pulls. This form of forced exploration was proposed in bootstrapping bandits earlier \cite{tang15personalized,elmachtoub17practical}, and is clearly not sound. A similar argument to \cref{lem:lower bound}, although less formal, is in Section 3.1 of \citet{osband15bootstrapped}.


\section{Garbage In, Reward Out}
\label{sec:giro}

One solution to the issues in \cref{sec:linear regret} is to add positive and negative pseudo rewards, $1$ and $0$, to $\cH_{i, s}$. This increases the variance of the bootstrap mean in \eqref{eq:follow the leader} and may lead to exploration. However, the pseudo rewards also introduce bias that has to be controlled. In the next section, we present a design that trades off these two quantities.

\subsection{Algorithm $\giro$}
\label{sec:algorithm giro}

We propose \cref{alg:giro}, which increases the variance of the bootstrap mean in \eqref{eq:follow the leader} by adding pseudo rewards to the history of the pulled arm (line \ref{alg:giro:garbage}). The algorithm is called $\giro$, which stands for \emph{garbage in, reward out}. This is an informal description of our exploration strategy, which adds seemingly useless extreme rewards to the history of the pulled arm. We call them \emph{pseudo rewards}, to distinguish them from observed rewards. $\giro$ has one tunable parameter $a$, the number of positive and negative pseudo rewards in the history for each observed reward.

$\giro$ does not seem sound, because the number of pseudo rewards in history $\cH_{i, s}$ grows linearly with $s$. In fact, this is the key idea in our design. We justify it informally in the next section and bound its regret in \cref{sec:analysis}.

\subsection{Informal Justification}
\label{sec:informal justification}

In this section, we informally justify the design of $\giro$ in a Bernoulli bandit. Our argument has two parts. First, we show that $\hat{\mu}_{i, t}$ in line \ref{alg:giro:sampling} \emph{concentrates} at the scaled and shifted expected reward of arm $i$, which preserves the order of the arms. Second, we show that $\hat{\mu}_{i, t}$ is \emph{optimistic}, its value is higher than the scaled and shifted expected reward of arm $i$, with a sufficiently high probability. This is sufficient for the regret analysis in \cref{sec:analysis}.

Fix arm $i$ and the number of its pulls $s$. Let $V_{i, s}$ denote the number of ones in history $\cH_{i, s}$, which includes $a$ positive and negative pseudo rewards for each observed reward of arm $i$. By definition, $V_{i, s} - a s$ is the number of positive observed rewards of arm $i$. These rewards are drawn i.i.d.\ from $\mathrm{Ber}(\mu_i)$. Thus $V_{i, s} - a s \sim B(s, \mu_i)$ and we have
\begin{align}
  \E{V_{i, s}}
  = (\mu_i + a) s\,, \quad
  \var{V_{i, s}}
  = \mu_i (1 - \mu_i) s\,.
  \label{eq:history moments}
\end{align}
Now we define $\alpha = 2 a + 1$ and let $U_{i, s}$ be the number of ones in bootstrap sample $\cB_{i, s}$. Since drawing $\alpha s$ samples with replacement from $\cH_{i, s}$ is equivalent to drawing $\alpha s$ i.i.d.\ samples from $\mathrm{Ber}(V_{i, s} / (\alpha s))$, we have $U_{i, s} \mid V_{i, s} \sim B(\alpha s, V_{i, s} / (\alpha s))$. From the definition of $U_{i, s}$, we have for any $V_{i, s}$,
\begin{align}
  \condE{U_{i, s}}{V_{i, s}}
  & = V_{i, s}\,,
  \label{eq:bootstrap mean} \\
  \condvar{U_{i, s}}{V_{i, s}}
  & = V_{i, s} \left(1 - \frac{V_{i, s}}{\alpha s}\right)\,.
  \label{eq:bootstrap variance}
\end{align}
Let $\hat{\mu} = U_{i, s} / (\alpha s)$ be the mean reward in $\cB_{i, s}$. First, we argue that $\hat{\mu}$ concentrates. From the properties of $U_{i, s}$ in \eqref{eq:bootstrap mean} and \eqref{eq:bootstrap variance}, for any $V_{i, s}$, we have
\begin{align*}
  \condE{\hat{\mu}}{V_{i, s}}
  = \frac{V_{i, s}}{\alpha s}\,, \quad
  \condvar{\hat{\mu}}{V_{i, s}}
  = \frac{V_{i, s}}{\alpha^2 s^2}
  \left(1 - \frac{V_{i, s}}{\alpha s}\right)\,.
\end{align*}
Since $V_{i, s} \in [a s, (a + 1) s]$, $\condvar{\hat{\mu}}{V_{i, s}} = O(1 / s)$ and $\hat{\mu} \to V_{i, s} / (\alpha s)$ as $s$ increases. Moreover, from the properties of $V_{i, s}$ in \eqref{eq:history moments}, we have
\begin{align*}
  \E{\frac{V_{i, s}}{\alpha s}}
  = \frac{\mu_i + a}{\alpha}\,, \quad
  \var{\frac{V_{i, s}}{\alpha s}}
  = \frac{\mu_i (1 - \mu_i)}{\alpha^2 s}\,.
\end{align*}
So, $V_{i, s} / (\alpha s) \to (\mu_i + a) / \alpha$ as $s$ increases. By transitivity, $\hat{\mu} \to (\mu_i + a) / \alpha$, which is the scaled and shifted expected reward of arm $i$. This transformation preserves the order of the arms; but it changes the gaps, the differences in the expected rewards of the optimal and suboptimal arms.

\begin{algorithm}[t]
  \caption{$\giro$ with $[0, 1]$ rewards.}
  \label{alg:giro}
  \begin{algorithmic}[1]
    \State \textbf{Inputs:} Pseudo rewards per unit of history $a$
    \Statex \vspace{-0.05in}
    \State $\forall i \in [K]: T_{i, 0} \gets 0, \, \cH_{i, 0} \gets ()$
    \Comment{Initialization}
    \For{$t = 1, \dots, n$}
      \For{$i = 1, \dots, K$}
      \Comment{Estimate arm values}
        \If{$T_{i, t - 1} > 0$}
          \State $s \gets T_{i, t - 1}$
          \State $\cB_{i, s} \gets
          \text{Sample $\abs{\cH_{i, s}}$ times from $\cH_{i, s}$}$
          \Statex \hspace{1.03in} with replacement
          \State $\displaystyle \hat{\mu}_{i, t} \gets
          \frac{1}{\abs{\cB_{i, s}}} \sum_{y \in \cB_{i, s}} y$
          \label{alg:giro:sampling}
        \Else
          \State $\hat{\mu}_{i, t} \gets + \infty$
          \label{alg:giro:initialization}
        \EndIf
      \EndFor
      \State $I_t \gets \argmax_{i \in [K]} \hat{\mu}_{i, t}$
      \Comment{Pulled arm}
      \State Pull arm $I_t$ and get reward $Y_{I_t, t}$
      \Statex \vspace{-0.05in}
      \For{$i = 1, \dots, K$}
      \Comment{Update statistics}
        \If{$i = I_t$}
          \State $T_{i, t} \gets T_{i, t - 1} + 1$
          \State $\cH_{i, T_{i, t}} \gets
          \cH_{i, T_{i, t - 1}} \oplus (Y_{i, t})$
          \For{$\ell = 1, \dots, a$}
          \Comment{Pseudo rewards}
            \State $\cH_{i, T_{i, t}} \gets
            \cH_{i, T_{i, t}} \oplus (0, 1)$
            \label{alg:giro:garbage}
          \EndFor
        \Else
          \State $T_{i, t} \gets T_{i, t - 1}$
        \EndIf
      \EndFor
    \EndFor
  \end{algorithmic}
\end{algorithm}

Second, we argue that $\hat{\mu}$ is \emph{optimistic}, that any unfavorable history is less likely than being optimistic under that history. In particular, let $E = \set{V_{i, s} / (\alpha s) = (\mu_i + a) / \alpha - \eps}$ be the event that the mean reward in the history deviates from its expectation by $\eps > 0$. Then
\begin{align}
  \condprob{\hat{\mu} \geq (\mu_i + a) / \alpha}{E}
  \geq \prob{E}
  \label{eq:informal optimism}
\end{align}
holds for any $\eps > 0$ such that $\prob{E} > 0$.

Trivially, $\prob{E} \leq \prob{V_{i, s} / (\alpha s) \leq (\mu_i + a) / \alpha - \eps}$ holds for any $\eps > 0$. Therefore, if both $V_{i, s} / (\alpha s)$ and $\hat{\mu} \mid V_{i, s}$ were normally distributed, inequality \eqref{eq:informal optimism} would hold if for any $V_{i, s}$, $\condvar{\hat{\mu}}{V_{i, s}} \geq \var{V_{i, s} / (\alpha s)}$. We compare the variances below. Since $V_{i, s} \in [a s, (a + 1) s]$,
\begin{align*}
  \condvar{\hat{\mu}}{V_{i, s}}
  & \geq \frac{1}{\alpha s} \min_{v \in [a s, (a + 1) s]}
  \frac{v}{\alpha s} \left(1 - \frac{v}{\alpha s}\right) \\
  & = \frac{1}{\alpha s} \frac{a s}{\alpha s} \left(1 - \frac{a s}{\alpha s}\right)
  = \frac{a (a + 1)}{\alpha^3 s}\,.
\end{align*}
Trivially, $\var{V_{i, s} / (\alpha s)} \leq 1 / (4 \alpha^2 s)$. Therefore, for any $a$ such that $a (a + 1) / \alpha \geq 1 / 4$, roughly $a \geq 1 / 3$, $\hat{\mu}$ is optimistic. We formalize this intuition in \cref{sec:analysis}. A formal proof is necessary because our assumption of normality was unrealistic.

\subsection{Contextual $\giro$}
\label{sec:contextual giro}

We generalize $\giro$ to a contextual bandit in \cref{alg:contextual giro}. The main difference in \cref{alg:contextual giro} is that it fits a reward generalization model to $\cB_{i, s}$ and then estimates the value of arm $i$ in context $x_t$ (line \ref{alg:contextual giro:estimate}) based on this model. If this model was linear with parameters $\theta_i$, the estimated value would be $x_t\transpose \theta_i$. The other difference is that the pseudo rewards are associated with context $x_t$ (line \ref{alg:contextual giro:garbage}), to increase the conditional variance of the estimates given $x_t$.

The value of arm $i$ in round $t$, $\condE{Y_{i, t}}{\cB_{i, s}, x_t}$, can be approximated by any function of $x_t$ that can be learned from $\cB_{i, s}$. The approximation should permit any constant shift of any representable function. In linear models, this can be achieved by adding a bias term to $x_t$. We experiment with multiple reward generalization models in \cref{sec:contextual experiments}.

\begin{algorithm}[t]
  \caption{Contextual $\giro$ with $[0, 1]$ rewards.}
  \label{alg:contextual giro}
  \begin{algorithmic}[1]
    \State \textbf{Inputs:} Pseudo rewards per unit of history $a$
    \Statex \vspace{-0.05in}
    \State $\forall i \in [K]: T_{i, 0} \gets 0, \, \cH_{i, 0} \gets ()$
    \Comment{Initialization}
    \For{$t = 1, \dots, n$}
      \For{$i = 1, \dots, K$}
      \Comment{Estimate arm values}
        \If{$T_{i, t - 1} > 0$}
          \State $s \gets T_{i, t - 1}$
          \State $\cB_{i, s} \gets
          \text{Sample $\abs{\cH_{i, s}}$ times from $\cH_{i, s}$}$
          \Statex \hspace{1.03in} with replacement
          \State $\hat{\mu}_{i, t} \gets
          \text{Estimate $\condE{Y_{i, t}}{\cB_{i, s}, x_t}$}$
          \label{alg:contextual giro:estimate}
        \Else
          \State $\hat{\mu}_{i, t} \gets + \infty$
        \EndIf
      \EndFor
      \State $I_t \gets \argmax_{i \in [K]} \hat{\mu}_{i, t}$
      \Comment{Pulled arm}
      \State Pull arm $I_t$ and get reward $Y_{I_t, t}$
      \Statex \vspace{-0.05in}
      \For{$i = 1, \dots, K$}
      \Comment{Update statistics}
        \If{$i = I_t$}
          \State $T_{i, t} \gets T_{i, t - 1} + 1$
          \State $\cH_{i, T_{i, t}} \gets
          \cH_{i, T_{i, t - 1}} \oplus ((x_t, Y_{i, t}))$
          \For{$\ell = 1, \dots, a$}
          \Comment{Pseudo rewards}
            \State $\cH_{i, T_{i, t}} \gets
            \cH_{i, T_{i, t}} \oplus ((x_t, 0), (x_t, 1))$
            \label{alg:contextual giro:garbage}
          \EndFor
        \Else
          \State $T_{i, t} \gets T_{i, t - 1}$
        \EndIf
      \EndFor
    \EndFor
  \end{algorithmic}
\end{algorithm}


\section{Analysis}
\label{sec:analysis}

In \cref{sec:gre analysis}, we prove an upper bound on the expected $n$-round regret of \cref{alg:gre}. In \cref{sec:giro analysis}, we prove an upper bound on the expected $n$-round regret of $\giro$ in a Bernoulli bandit. In \cref{sec:discussion}, we discuss the results of our analysis. Note that our analysis is in the multi-armed bandit setting.

\subsection{General Randomized Exploration}
\label{sec:gre analysis}

We prove an upper bound on the regret of \cref{alg:gre} in a multi-armed bandit with $K$ arms below. The setting and regret are formally defined in \cref{sec:setting}. The distribution $p$ in line \ref{alg:gre:estimate} of \cref{alg:gre} is a function of the history of the arm. For $s \in [n] \cup \set{0}$, let
\begin{align}
  Q_{i, s}(\tau)
  = \condprob{\hat{\mu} \geq \tau}{\hat{\mu} \sim p(\cH_{i, s}), \, \cH_{i, s}}
  \label{eq:optimism}
\end{align}
be the tail probability that $\hat{\mu}$ conditioned on history $\cH_{i, s}$ is at least $\tau$ for some tunable parameter $\tau$.

\begin{theorem}
\label{thm:gre upper bound} For any tunable parameters $(\tau_i)_{i = 2}^K \in \realset^{K - 1}$, the expected $n$-round regret of \cref{alg:gre} can be bounded from above as $R(n) \leq \sum_{i = 2}^K \Delta_i (a_i + b_i)$, where
\begin{align*}
  a_i
  & = \sum_{s = 0}^{n - 1} \E{\min \set{1 / Q_{1, s}(\tau_i) - 1, n}}\,, \\
  b_i
  & = \sum_{s = 0}^{n - 1} \prob{Q_{i, s}(\tau_i) > 1 / n} + 1\,.
\end{align*}
\end{theorem}
\begin{proof}
A detailed proof is in \cref{sec:gre proof}. It is motivated by the proof of Thompson sampling \cite{agrawal13further}. Our main contribution is that we state and prove the claim such that it can be reused for the regret analysis of any sampling distribution in \cref{alg:gre}.
\end{proof}

\cref{thm:gre upper bound} says that the regret of \cref{alg:gre} is low when $Q_{i, s}(\tau_i) \to 0$ and $Q_{1, s}(\tau_i) \to 1$ as $s \to \infty$. This suggests the following setting of the tunable parameter $\tau_i$ for $i > 1$. When $p(\cH_{i, s})$ concentrates at $\mu_i$ and $\mu_i < \mu_1$, $\tau_i$ should be chosen from interval $(\mu_i, \mu_1)$. Then $Q_{i, s}(\tau_i) \to 0$ and $Q_{1, s}(\tau_i) \to 1$ would follow by concentration as $s \to \infty$. In general, when $p(\cH_{i, s})$ concentrates at $\mu_i'$ and $\mu_i' < \mu_1'$, $\tau_i$ should be chosen from interval $(\mu_i', \mu_1')$.

\subsection{Bernoulli $\giro$}
\label{sec:giro analysis}

We analyze $\giro$ in a $K$-armed Bernoulli bandit. We make an additional assumption over \cref{sec:gre analysis} that the rewards are binary. Our regret bound is stated below. 

\begin{theorem}
\label{thm:giro upper bound} For any $a > 1 / \sqrt{2}$, the expected $n$-round regret of $\giro$ is bounded from above as
\begin{align}
  R(n)
  \leq \sum_{i = 2}^K \Delta_i
  \Bigg[& \underbrace{\left(\frac{16 (2 a + 1) c}{\Delta_i^2} \log n + 2\right)}_
  {\text{\emph{Upper bound on $a_i$ in \cref{thm:gre upper bound}}}} + {}
  \label{eq:upper bound} \\
  & \underbrace{\left(\frac{8 (2 a + 1)}{\Delta_i^2} \log n + 2\right)}_
  {\text{\emph{Upper bound on $b_i$ in \cref{thm:gre upper bound}}}}\Bigg]\,,
  \nonumber
\end{align}
where $b = (2 a + 1) / [a (a + 1)]$ and
\begin{align*}
  c
  = \frac{2 e^2 \sqrt{2 a + 1}}{\sqrt{2 \pi}} \exp\left[\frac{8 b}{2 - b}\right]
  \left(1 + \sqrt{\frac{2 \pi}{4 - 2 b}}\right)
  \nonumber
\end{align*}
is an upper bound on the expected inverse probability of being optimistic, which is derived in \cref{sec:optimism upper bound}.

\end{theorem}
\begin{proof}
The claim is proved in \cref{sec:giro proof}. The key steps in the analysis are outlined below.
\end{proof}

Since $\giro$ is an instance of \cref{alg:gre}, we prove \cref{thm:giro upper bound} using \cref{thm:gre upper bound}, where we instantiate distribution specific artifacts. In the notation of \cref{sec:gre analysis}, $Q_{i, s}(\tau)$ in \eqref{eq:optimism} is the probability that the mean reward in the bootstrap sample $\cB_{i, s}$ of arm $i$ after $s$ pulls is at least $\tau$, conditioned on history $\cH_{i, s}$.

Fix any suboptimal arm $i$. Based on \cref{sec:informal justification}, the mean reward in $\cB_{i, s}$ concentrates at $\mu_i' = (\mu_i + a) / \alpha$, where $\alpha = 2 a + 1$. Following the discussion on the choice of $\tau_i$ in \cref{sec:gre analysis}, we set $\tau_i = (\mu_i + a) / \alpha + \Delta_i / (2 \alpha)$, which is the average of $\mu_i'$ and $\mu_1'$. Recall that $V_{i, s}$ is the number of ones in history $\cH_{i, s}$ and $U_{i, s}$ is the number of ones in its bootstrap sample $\cB_{i, s}$. Then $Q_{i, s}(\tau_i)$ can be written as
\begin{align*}
  Q_{i, s}(\tau_i)
  & = \condprob{\frac{U_{i, s}}{\alpha s} \geq
  \frac{\mu_i + a}{\alpha} + \frac{\Delta_i}{2 \alpha}}{V_{i, s}}
  \text{ for } s > 0\,, \\
  Q_{i, 0}(\tau_i)
  & = 1\,,
\end{align*}
where $Q_{i, 0}(\tau_i) = 1$ because of the initialization in line \ref{alg:giro:initialization} of $\giro$. We define $Q_{1, s}(\tau_i)$ analogously, by replacing $U_{i, s}$ with $U_{1, s}$ and $V_{i, s}$ with $V_{1, s}$.

The bound in \cref{thm:giro upper bound} is proved as follows. To simplify notation, we introduce $a_{i, s} = \E{\min \set{1 / Q_{1, s}(\tau_i) - 1, n}}$ and $b_{i, s} = \prob{Q_{i, s}(\tau_i) > 1 / n}$. By concentration, $a_{i, s}$ and $b_{i, s}$ are small when the number of pulls $s$ is large, on the order of $\Delta_i^{-2} \log n$. When the number of pulls $s$ is small, we bound $b_{i, s}$ trivially by $1$ and $a_{i, s}$ as
\begin{align*}
  a_{i, s}
  \leq \E{1 / \condprob{U_{1, s} \geq (\mu_1 + a) s}{V_{1, s}}}\,,
\end{align*}
where the right-hand side is bounded in \cref{sec:optimism upper bound}. The analysis in \cref{sec:optimism upper bound} is novel and shows that sampling from a binomial distribution with pseudo rewards is sufficiently optimistic.

\subsection{Discussion}
\label{sec:discussion}

The regret of $\giro$ is bounded from above in \cref{thm:giro upper bound}. Our bound is $O(K \Delta^{-1} \log n)$, where $K$ is the number of arms, $\Delta = \min_{i > 1} \Delta_i$ is the minimum gap, and $n$ is the number of rounds. The bound matches the regret bound of $\ucb$ in all quantities of interest.

We would like to discuss constants $a$ and $c$ in \cref{thm:giro upper bound}. The bound increases with the number of pseudo rewards $a$. This is expected, because $a$ positive and negative pseudo rewards yield $2 a + 1$ times smaller gaps than in the original problem (\cref{sec:informal justification}). The benefit is that exploration becomes easy. Although the gaps are smaller, $\giro$ outperforms $\ucb$ in all experiments in \cref{sec:mab experiments}, even for $a = 1$. The experiments also show that the regret of $\giro$ increases with $a$, as suggested by our bound.

The constant $c$ in \cref{thm:giro upper bound} is defined for all $a > 1 / \sqrt{2}$ and can be large due to the term $f(b) = \exp[8 b/ (2 - b)]$. For instance, for $a = 1$, $b = 3 / 2$ and $f(b) \approx e^{24}$. Fortunately, $f(b)$ decreases quickly as $a$ increases. For $a = 2$, $b = 5 / 6$ and $f(b) \approx e^{5.7}$; and for $a = 3$, $b = 7 / 12$ and $f(b) \approx e^{3.3}$. Because $\giro$ performs well empirically, as shown in \cref{sec:mab experiments}, the theoretically-suggested value of $c$ is likely to be loose.


\section{Experiments}
\label{sec:experiments}

We conduct two experiments. In \cref{sec:mab experiments}, we evaluate $\giro$ on multi-armed bandit problems. In \cref{sec:contextual experiments}, we evaluate $\giro$ in the contextual bandit setting.

\begin{figure*}[t]
  \centering
  \includegraphics[width=6.75in]{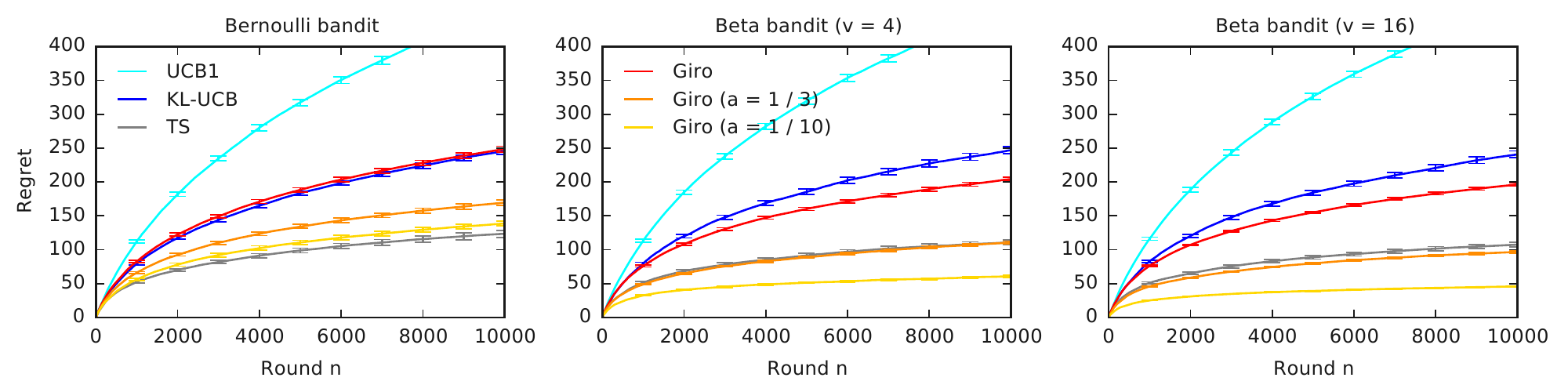} \\
  \vspace{-0.1in}
  \caption{Comparison of $\giro$ to $\ucb$, $\klucb$, and $\ts$ on three multi-armed bandit problems in \cref{sec:mab experiments}. The regret is reported as a function of round $n$. The results are averaged over $100$ runs. To reduce clutter, the legend is split between the first two plots.}
  \label{fig:mab results}
\end{figure*}

\begin{figure*}[t]
  \centering
  \includegraphics[width=6.75in]{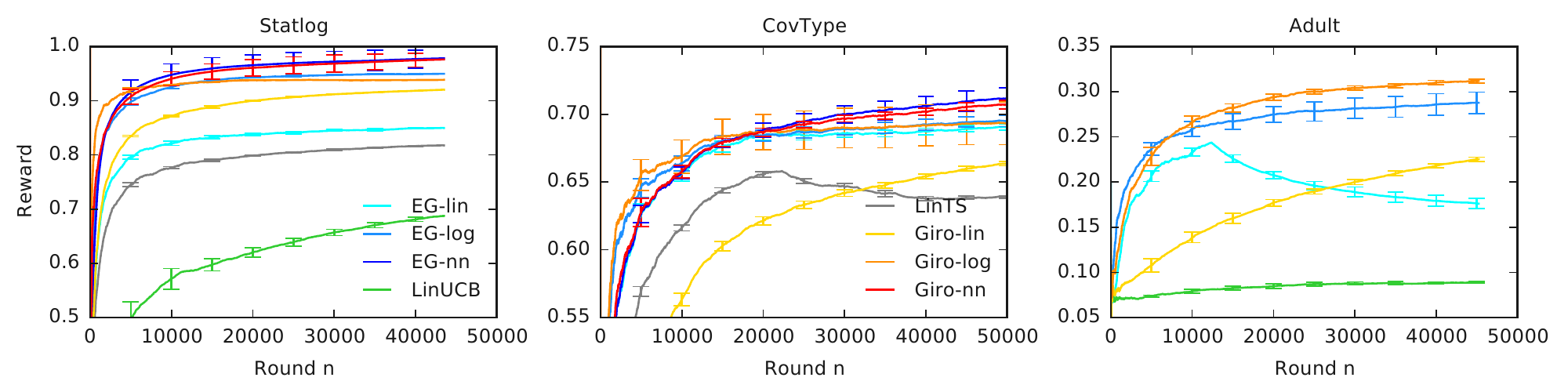} \\
  \vspace{-0.1in}
  \caption{Comparison of $\giro$ to $\linucb$, $\lints$, and the $\epsilon$-greedy policy on three contextual problems in \cref{sec:contextual experiments}. The reward is reported as a function of round $n$. The results are averaged over $5$ runs. To reduce clutter, the legend is split between the first two plots.}
\label{fig:contextual results}
\end{figure*}

\subsection{Multi-Armed Bandit}
\label{sec:mab experiments}

We run $\giro$ with three different values of $a$: $1$, $1 / 3$, and $1 / 10$. Our regret analysis in \cref{sec:discussion} justifies the value of $a = 1$. The informal argument in \cref{sec:informal justification} suggests a less conservative value of $a = 1 / 3$. We implement $\giro$ with real $a > 0$ as follows. For each arm $i$, we have two histories after $s$ pulls, with $\floors{a s}$ and $\ceils{a s}$ pseudo rewards of each kind. The value of $\hat{\mu}_{i, t}$ is estimated from the $\ceils{a s}$ and $\floors{a s}$ histories with probabilities $a s - \floors{a s}$ and $\ceils{a s} - a s$, respectively. This interpolation seems natural.

$\giro$ is compared to $\ucb$ \cite{auer02finitetime}, Bernoulli $\ts$ \cite{agrawal13further}, and $\klucb$ \cite{garivier11klucb}. The prior in Bernoulli $\ts$ is $\mathrm{Beta}(1, 1)$. We implement $\ts$ and $\klucb$ with $[0, 1]$ rewards as described in \citet{agrawal13further}. Specifically, for any $Y_{i, t} \in [0, 1]$, we draw pseudo reward $\hat{Y}_{i, t} \sim \mathrm{Ber}(Y_{i, t})$ and use it instead of $Y_{i, t}$.

We experiment with two classes of $K$-armed bandit problems where the reward distribution $P_i$ of arm $i$ is parameterized by its expected reward $\mu_i \in [0, 1]$. The first class is a Bernoulli bandit, where $P_i = \mathrm{Ber}(\mu_i)$. Both $\klucb$ and $\ts$ are near optimal in this class. The second class is a beta bandit, where $P_i = \mathrm{Beta}(v \mu_i, v (1 - \mu_i))$ for $v \geq 1$. Both $\klucb$ and $\ts$ can solve such problems, but are not statistically optimal anymore. We experiment with $v = 4$, which leads to rewards of higher variances; and $v = 16$, which leads to rewards of lower variances. The number of arms is $K = 10$ and their means are chosen uniformly at random from $[0.25, 0.75]$. The horizon is $n = 10$k rounds.

Our results are reported in \cref{fig:mab results}. The strong empirical performance of $\giro$ is apparent. $\giro$ outperforms $\ucb$ in all problems and for all values of $a$. It also outperforms $\klucb$ in the beta bandit for all values of $a$. Finally, $\giro$ outperforms $\ts$ in the beta bandit for $a = 1 / 10$. Although this setting of $\giro$ is purely heuristic, it shows the potential of our proposed method.

The goal of our experiments was to show that $\giro$ performs well in general, not that it outperforms near-optimal algorithms for well-established classes of bandit problems. $\giro$ is meant to be general, as shown in \cref{sec:contextual experiments}.

\subsection{Contextual Bandit}
\label{sec:contextual experiments}

We conduct contextual bandit experiments on multi-class classification problems \cite{agarwal14taming,elmachtoub17practical,riquelme18deep}, where arm $i \in [K]$ corresponds to class $i$. In round $t$, the algorithm observes context $x_t \in \realset^d$ and pulls an arm. It receives a reward of one if the pulled arm is the correct class, and zero otherwise. Each arm maintains independent statistics that map context $x_t$ to a binary reward. We use three datasets from \citet{riquelme18deep}: Adult ($d = 94, K = 14$), Statlog ($d = 9, K = 7$), and CovType ($d = 54, K = 7$).

The horizon is $n = 50$k rounds and we average our results over $5$ runs. $\giro$ is compared to $\linucb$ \cite{abbasi-yadkori11improved}, $\lints$ \cite{agrawal13thompson}, and the $\epsilon$-greedy policy ($\egreedy$) \cite{auer02finitetime}. We also implemented $\ucbglm$ \cite{li17provably}. $\ucbglm$ over-explored and performed worse than our other baselines. Therefore, we do not report its results in this paper.

We experiment with three reward generalization models in $\giro$ and $\egreedy$: linear (suffix \emph{lin} in plots), logistic (suffix \emph{log} in plots), and a single hidden-layer fully-connected neural network (suffix \emph{nn} in plots) with ten hidden neurons. We experimented with different exploration schedules for $\egreedy$. The best schedule across all datasets was $\epsilon_t = b / t$, where $b$ is set to attain $1\%$ exploration in $n$ rounds. Note that this tuning gives $\egreedy$ an unfair advantage over other algorithms. We tune $\egreedy$ because it performs poorly without tuning.

In $\giro$, $a = 1$ in all experiments. The parameters of the reward generalization model (line \ref{alg:contextual giro:estimate} in \cref{alg:contextual giro}) are fit in each round using maximum likelihood estimation. We solve the problem using stochastic optimization, which is initialized by the solution in the previous round. In linear and logistic models, we optimize until the error drops below $10^{-3}$. In neural networks, we make one pass over the whole history. To ensure reproducibility of our results, we use public optimization libraries. For linear and logistic models, we use scikit-learn \cite{pedregosa11scikitlearn} with stochastic optimization and its default settings. For neural networks, we use Keras \cite{chollet15keras} with a ReLU hidden layer and a sigmoid output layer, along with SGD and its default settings. In comparison, \citet{elmachtoub17practical} and \citet{tang15personalized} approximate bootstrapping by an ensemble of models. In general, our approach yields similar results to \citet{elmachtoub17practical} at a lower computational cost, and better results than \citet{tang15personalized} without any tuning.

Since we compare different bandit algorithms and reward generalization models, we use the expected per-round reward in $n$ rounds, $\E{\sum_{t = 1}^n Y_{I_t, t}} / n$, as our metric. We report it for all algorithms in all datasets in \cref{fig:contextual results}. We observe the following trends. First, both linear methods, $\lints$ and $\linucb$, perform the worst.\footnote{$\linucb$ and $\lints$ are the worst performing methods in the CovType and Adult datasets, respectively. Since we want to show most rewarding parts of the plots, we do not plot these results.} Second, linear $\giro$ and $\egreedy$ are comparable in the Statlog and CovType datasets. In the Adult dataset, $\egreedy$ does not explore enough for the relatively larger number of arms. In contrast, $\giro$ explores enough and performs well. Third, the non-linear variants of $\egreedy$ and $\giro$ generally outperform their linear counterparts. The most expressive model, the neural network, outperforms the logistic model in both the Statlog and CovType datasets. In the Adult dataset, the neural network performs the worst and we do not plot it. To investigate this further, we trained a neural network offline for each arm with all available data. Even then, the neural network performed worse than a linear model. We conclude that the poor performance of neural networks is caused by poor generalization and not the lack of exploration.


\section{Related Work}
\label{sec:related work}

\citet{osband15bootstrapped} proposed a bandit algorithm, which they call $\bootstrapts$; that pulls the arm with the highest bootstrap mean, which is estimated from a history with pseudo rewards. They also showed in a Bernoulli bandit that $\bootstrapts$ is equivalent to Thompson sampling. \citet{vaswani18new} generalized this result to categorical and Gaussian rewards. In relation to these works, we make the following contributions. First, $\giro$ is an instance of $\bootstrapts$ with non-parametric bootstrapping. The novelty in $\giro$ is in the design of pseudo rewards, which are added after each pull of the arm and have extreme values. Second, our regret analysis of $\giro$ is the first proof that justifies the use of non-parametric bootstrapping, the most common form of bootstrapping, for exploration. Finally, \cref{alg:gre} is more general than $\bootstrapts$. Its analysis in \cref{sec:gre analysis} shows that randomization alone, not necessarily by posterior sampling, induces exploration.

\citet{eckles14thompson} approximated the posterior of each arm in a multi-armed bandit by multiple bootstrap samples of its history. In round $t$, the agent chooses randomly one sample per arm and then pulls the arm with the highest mean reward in that sample. The observed reward is added with probability $0.5$ to all samples of the history. A similar method was proposed in contextual bandits by \citet{tang15personalized}. The key difference is that the observed reward is added to all samples of the history with random Poisson weights that control its importance. \citet{elmachtoub17practical} proposed bootstrapping with decision trees in contextual bandits. \citet{tang15personalized} and \citet{elmachtoub17practical} also provided limited theoretical justification for bootstrapping as a form of posterior sampling. But this justification is not strong enough to derive regret bounds. We prove regret bounds and do not view bootstrapping as an approximation to posterior sampling.

\citet{baransi14subsampling} proposed a sampling technique that equalizes the histories of arms in a $2$-armed bandit. Let $n_1$ and $n_2$ be the number of rewards of arms $1$ and $2$, respectively. Let $n_1 < n_2$. Then the value of arm $1$ is estimated by its empirical mean and the value of arm $2$ is estimated by the empirical mean in the bootstrap sample of its history of size $n_1$. \citet{baransi14subsampling} bounded the regret of their algorithm and \citet{osband15bootstrapped} showed empirically that its regret can be linear.


\section{Conclusions}
\label{sec:conclusions}

We propose $\giro$, a novel bandit algorithm that pulls the arm with the highest mean reward in a non-parametric bootstrap sample of its history with pseudo rewards. The pseudo rewards are designed such that the bootstrap mean is optimistic with a high probability. We analyze $\giro$ and bound its $n$-round regret. This is the first formal proof that justifies the use of non-parametric bootstrapping, the most common form of bootstrapping, for exploration. $\giro$ can be easily applied to structured problems, and we evaluate it on both synthetic and real-world problems.

Our upper bound on the regret of randomized exploration in \cref{thm:gre upper bound} is very general, and says that any algorithm that controls the tails of the sampling distribution $p$ in \cref{alg:gre} has low regret. This shows that any appropriate randomization, not necessarily posterior sampling, can be used for exploration. In future work, we plan to investigate other randomized algorithms that easily generalize to complex problems, such as pull the arm with the highest empirical mean with randomized pseudo rewards \cite{kveton19perturbed}. We believe that the key ideas in the proof of \cref{thm:gre upper bound} can be generalized to structured problems, such as linear bandits \cite{kveton19perturbedlinear}.

History resampling is computationally expensive, because the history of the arm has to be resampled in each round. This can be avoided in some cases. In a Bernoulli bandit, $\giro$ can be implemented efficiently, because the value of the arm can be drawn from a binomial distribution, as discussed in \cref{sec:informal justification}. A natural computationally-efficient substitute for resampling is an ensemble of fixed perturbations of the history, as in ensemble sampling \cite{lu17ensemble}. We plan to investigate it in future work.

\bibliographystyle{plainnat}
\bibliography{References}

\begin{thebibliography}{33}
\providecommand{\natexlab}[1]{#1}
\providecommand{\url}[1]{\texttt{#1}}
\expandafter\ifx\csname urlstyle\endcsname\relax
  \providecommand{\doi}[1]{doi: #1}\else
  \providecommand{\doi}{doi: \begingroup \urlstyle{rm}\Url}\fi

\bibitem[Abbasi-Yadkori et~al.(2011)Abbasi-Yadkori, Pal, and
  Szepesvari]{abbasi-yadkori11improved}
Yasin Abbasi-Yadkori, David Pal, and Csaba Szepesvari.
\newblock Improved algorithms for linear stochastic bandits.
\newblock In \emph{Advances in Neural Information Processing Systems 24}, pages
  2312--2320, 2011.

\bibitem[Agarwal et~al.(2014)Agarwal, Hsu, Kale, Langford, Li, and
  Schapire]{agarwal14taming}
Alekh Agarwal, Daniel Hsu, Satyen Kale, John Langford, Lihong Li, and Robert
  Schapire.
\newblock Taming the monster: A fast and simple algorithm for contextual
  bandits.
\newblock In \emph{Proceedings of the 31st International Conference on Machine
  Learning}, pages 1638--1646, 2014.

\bibitem[Agrawal and Goyal(2013{\natexlab{a}})]{agrawal13further}
Shipra Agrawal and Navin Goyal.
\newblock Further optimal regret bounds for {Thompson} sampling.
\newblock In \emph{Proceedings of the 16th International Conference on
  Artificial Intelligence and Statistics}, pages 99--107, 2013{\natexlab{a}}.

\bibitem[Agrawal and Goyal(2013{\natexlab{b}})]{agrawal13thompson}
Shipra Agrawal and Navin Goyal.
\newblock Thompson sampling for contextual bandits with linear payoffs.
\newblock In \emph{Proceedings of the 30th International Conference on Machine
  Learning}, pages 127--135, 2013{\natexlab{b}}.

\bibitem[Auer et~al.(2002)Auer, Cesa-Bianchi, and Fischer]{auer02finitetime}
Peter Auer, Nicolo Cesa-Bianchi, and Paul Fischer.
\newblock Finite-time analysis of the multiarmed bandit problem.
\newblock \emph{Machine Learning}, 47:\penalty0 235--256, 2002.

\bibitem[Baransi et~al.(2014)Baransi, Maillard, and
  Mannor]{baransi14subsampling}
Akram Baransi, Odalric-Ambrym Maillard, and Shie Mannor.
\newblock Sub-sampling for multi-armed bandits.
\newblock In \emph{Proceeding of European Conference on Machine Learning and
  Principles and Practice of Knowledge Discovery in Databases}, 2014.

\bibitem[Chollet et~al.(2015)]{chollet15keras}
Francois Chollet et~al.
\newblock Keras.
\newblock \url{https://keras.io}, 2015.

\bibitem[Dani et~al.(2008)Dani, Hayes, and Kakade]{dani08stochastic}
Varsha Dani, Thomas Hayes, and Sham Kakade.
\newblock Stochastic linear optimization under bandit feedback.
\newblock In \emph{Proceedings of the 21st Annual Conference on Learning
  Theory}, pages 355--366, 2008.

\bibitem[Doob(1953)]{doob53stochastic}
Joseph Doob.
\newblock \emph{Stochastic Processes}.
\newblock John Wiley \& Sons, 1953.

\bibitem[Eckles and Kaptein(2014)]{eckles14thompson}
Dean Eckles and Maurits Kaptein.
\newblock Thompson sampling with the online bootstrap.
\newblock \emph{CoRR}, abs/1410.4009, 2014.
\newblock URL \url{http://arxiv.org/abs/1410.4009}.

\bibitem[Efron and Tibshirani(1986)]{efron86bootstrap}
Bradley Efron and Robert Tibshirani.
\newblock Bootstrap methods for standard errors, confidence intervals, and
  other measures of statistical accuracy.
\newblock \emph{Statistical Science}, 1\penalty0 (1):\penalty0 54--75, 1986.

\bibitem[Elmachtoub et~al.(2017)Elmachtoub, McNellis, Oh, and
  Petrik]{elmachtoub17practical}
Adam Elmachtoub, Ryan McNellis, Sechan Oh, and Marek Petrik.
\newblock A practical method for solving contextual bandit problems using
  decision trees.
\newblock In \emph{Proceedings of the 33rd Conference on Uncertainty in
  Artificial Intelligence}, 2017.

\bibitem[Filippi et~al.(2010)Filippi, Cappe, Garivier, and
  Szepesvari]{filippi10parametric}
Sarah Filippi, Olivier Cappe, Aurelien Garivier, and Csaba Szepesvari.
\newblock Parametric bandits: The generalized linear case.
\newblock In \emph{Advances in Neural Information Processing Systems 23}, pages
  586--594, 2010.

\bibitem[Garivier and Cappe(2011)]{garivier11klucb}
Aurelien Garivier and Olivier Cappe.
\newblock The {KL-UCB} algorithm for bounded stochastic bandits and beyond.
\newblock In \emph{Proceeding of the 24th Annual Conference on Learning
  Theory}, pages 359--376, 2011.

\bibitem[Gopalan et~al.(2014)Gopalan, Mannor, and Mansour]{gopalan14thompson}
Aditya Gopalan, Shie Mannor, and Yishay Mansour.
\newblock Thompson sampling for complex online problems.
\newblock In \emph{Proceedings of the 31st International Conference on Machine
  Learning}, pages 100--108, 2014.

\bibitem[Jun et~al.(2017)Jun, Bhargava, Nowak, and Willett]{jun17scalable}
Kwang-Sung Jun, Aniruddha Bhargava, Robert Nowak, and Rebecca Willett.
\newblock Scalable generalized linear bandits: Online computation and hashing.
\newblock In \emph{Advances in Neural Information Processing Systems 30}, pages
  98--108, 2017.

\bibitem[Kawale et~al.(2015)Kawale, Bui, Kveton, Tran-Thanh, and
  Chawla]{kawale15efficient}
Jaya Kawale, Hung Bui, Branislav Kveton, Long Tran-Thanh, and Sanjay Chawla.
\newblock Efficient {Thompson} sampling for online matrix-factorization
  recommendation.
\newblock In \emph{Advances in Neural Information Processing Systems 28}, pages
  1297--1305, 2015.

\bibitem[Kveton et~al.(2019{\natexlab{a}})Kveton, Szepesvari, Ghavamzadeh, and
  Boutilier]{kveton19perturbed}
Branislav Kveton, Csaba Szepesvari, Mohammad Ghavamzadeh, and Craig Boutilier.
\newblock Perturbed-history exploration in stochastic multi-armed bandits.
\newblock In \emph{Proceedings of the 28th International Joint Conference on
  Artificial Intelligence}, 2019{\natexlab{a}}.

\bibitem[Kveton et~al.(2019{\natexlab{b}})Kveton, Szepesvari, Ghavamzadeh, and
  Boutilier]{kveton19perturbedlinear}
Branislav Kveton, Csaba Szepesvari, Mohammad Ghavamzadeh, and Craig Boutilier.
\newblock Perturbed-history exploration in stochastic linear bandits.
\newblock In \emph{Proceedings of the 35th Conference on Uncertainty in
  Artificial Intelligence}, 2019{\natexlab{b}}.

\bibitem[Lai and Robbins(1985)]{lai85asymptotically}
T.~L. Lai and Herbert Robbins.
\newblock Asymptotically efficient adaptive allocation rules.
\newblock \emph{Advances in Applied Mathematics}, 6\penalty0 (1):\penalty0
  4--22, 1985.

\bibitem[Lattimore and Szepesvari(2019)]{lattimore19bandit}
Tor Lattimore and Csaba Szepesvari.
\newblock \emph{Bandit Algorithms}.
\newblock Cambridge University Press, 2019.

\bibitem[Li et~al.(2010)Li, Chu, Langford, and Schapire]{li10contextual}
Lihong Li, Wei Chu, John Langford, and Robert Schapire.
\newblock A contextual-bandit approach to personalized news article
  recommendation.
\newblock In \emph{Proceedings of the 19th International Conference on World
  Wide Web}, 2010.

\bibitem[Li et~al.(2017)Li, Lu, and Zhou]{li17provably}
Lihong Li, Yu~Lu, and Dengyong Zhou.
\newblock Provably optimal algorithms for generalized linear contextual
  bandits.
\newblock In \emph{Proceedings of the 34th International Conference on Machine
  Learning}, pages 2071--2080, 2017.

\bibitem[Lu and {Van Roy}(2017)]{lu17ensemble}
Xiuyuan Lu and Benjamin {Van Roy}.
\newblock Ensemble sampling.
\newblock In \emph{Advances in Neural Information Processing Systems 30}, pages
  3258--3266, 2017.

\bibitem[Osband and {Van Roy}(2015)]{osband15bootstrapped}
Ian Osband and Benjamin {Van Roy}.
\newblock Bootstrapped {Thompson} sampling and deep exploration.
\newblock \emph{CoRR}, abs/1507.00300, 2015.
\newblock URL \url{http://arxiv.org/abs/1507.00300}.

\bibitem[Pedregosa et~al.(2011)Pedregosa, Varoquaux, Gramfort, Michel, Thirion,
  Grisel, Blondel, Prettenhofer, Weiss, Dubourg, VanderPlas, Passos,
  Cournapeau, Brucher, Perrot, and Duchesnay]{pedregosa11scikitlearn}
Fabian Pedregosa, Gael Varoquaux, Alexandre Gramfort, Vincent Michel, Bertrand
  Thirion, Olivier Grisel, Mathieu Blondel, Peter Prettenhofer, Ron Weiss,
  Vincent Dubourg, Jake VanderPlas, Alexandre Passos, David Cournapeau,
  Matthieu Brucher, Matthieu Perrot, and Edouard Duchesnay.
\newblock Scikit-learn: Machine learning in {Python}.
\newblock \emph{Journal of Machine Learning Research}, 12:\penalty0 2825--2830,
  2011.

\bibitem[Riquelme et~al.(2018)Riquelme, Tucker, and Snoek]{riquelme18deep}
Carlos Riquelme, George Tucker, and Jasper Snoek.
\newblock Deep {Bayesian} bandits showdown: An empirical comparison of
  {Bayesian} deep networks for {Thompson} sampling.
\newblock In \emph{Proceedings of the 6th International Conference on Learning
  Representations}, 2018.

\bibitem[Rusmevichientong and Tsitsiklis(2010)]{rusmevichientong10linearly}
Paat Rusmevichientong and John Tsitsiklis.
\newblock Linearly parameterized bandits.
\newblock \emph{Mathematics of Operations Research}, 35\penalty0 (2):\penalty0
  395--411, 2010.

\bibitem[Sutton and Barto(1998)]{sutton98reinforcement}
Richard Sutton and Andrew Barto.
\newblock \emph{Reinforcement Learning: An Introduction}.
\newblock MIT Press, Cambridge, MA, 1998.

\bibitem[Tang et~al.(2015)Tang, Jiang, Li, Zeng, and Li]{tang15personalized}
Liang Tang, Yexi Jiang, Lei Li, Chunqiu Zeng, and Tao Li.
\newblock Personalized recommendation via parameter-free contextual bandits.
\newblock In \emph{Proceedings of the 38th International ACM SIGIR Conference
  on Research and Development in Information Retrieval}, pages 323--332, 2015.

\bibitem[Thompson(1933)]{thompson33likelihood}
William~R. Thompson.
\newblock On the likelihood that one unknown probability exceeds another in
  view of the evidence of two samples.
\newblock \emph{Biometrika}, 25\penalty0 (3-4):\penalty0 285--294, 1933.

\bibitem[Vaswani et~al.(2018)Vaswani, Kveton, Wen, Rao, Schmidt, and
  Abbasi-Yadkori]{vaswani18new}
Sharan Vaswani, Branislav Kveton, Zheng Wen, Anup Rao, Mark Schmidt, and Yasin
  Abbasi-Yadkori.
\newblock New insights into bootstrapping for bandits.
\newblock \emph{CoRR}, abs/1805.09793, 2018.
\newblock URL \url{http://arxiv.org/abs/1805.09793}.

\bibitem[Zhang et~al.(2016)Zhang, Yang, Jin, Xiao, and Zhou]{zhang16online}
Lijun Zhang, Tianbao Yang, Rong Jin, Yichi Xiao, and Zhi-Hua Zhou.
\newblock Online stochastic linear optimization under one-bit feedback.
\newblock In \emph{Proceedings of the 33rd International Conference on Machine
  Learning}, pages 392--401, 2016.

\end{thebibliography}

\clearpage
\onecolumn
\appendix


\section{Proof of \cref{thm:gre upper bound}}
\label{sec:gre proof}

We generalize the analysis of \citet{agrawal13further}. Since arm $1$ is optimal, the regret can be written as
\begin{align*}
  R(n) =
  \sum_{i = 2}^K \Delta_i \E{T_{i, n}}\,.
\end{align*}
In the rest of the proof, we bound $\E{T_{i, n}}$ for each suboptimal arm $i$. Fix arm $i > 1$. Let $E_{i, t} = \set{\hat{\mu}_{i, t} \leq \tau_i}$ and $\bar{E}_{i, t}$ be the complement of $E_{i, t}$. Then $\E{T_{i, n}}$ can be decomposed as
\begin{align}
  \E{T_{i, n}} =
  \E{\sum_{t = 1}^n \I{I_t = i}} =
  \E{\sum_{t = 1}^n \I{I_t = i, E_{i, t} \text{ occurs}}} +
  \E{\sum_{t = 1}^n \I{I_t = i, \bar{E}_{i, t} \text{ occurs}}}\,.
  \label{eq:midpoint decomposition}
\end{align}

\subsubsection*{Term $b_i$ in the Upper Bound}

We start with the second term in \eqref{eq:midpoint decomposition}, which corresponds to $b_i$ in our claim. This term can be tightly bounded based on the observation that event $\bar{E}_{t, i}$ is unlikely when $T_{i, t}$ is \say{large}. Let $\mathcal{T} = \set{t \in [n]: Q_{i, T_{i, t - 1}}(\tau_i) > 1 / n}$. Then
\begin{align*}
  \E{\sum_{t = 1}^n \I{I_t = i, \bar{E}_{i, t} \text{ occurs}}}
  & \leq \E{\sum_{t \in \mathcal{T}} \I{I_t = i}} +
  \E{\sum_{t \not\in \mathcal{T}} \I{\bar{E}_{i, t}}} \\
  & \leq \E{\sum_{s = 0}^{n - 1} \I{Q_{i, s}(\tau_i) > 1 / n}} +
  \E{\sum_{t \not\in \mathcal{T}} \frac{1}{n}} \\
  & \leq \sum_{s = 0}^{n - 1} \prob{Q_{i, s}(\tau_i) > 1 / n} + 1\,.
\end{align*}

\subsubsection*{Term $a_i$ in the Upper Bound}

Now we focus on the first term in \eqref{eq:midpoint decomposition}, which corresponds to $a_i$ in our claim. Without loss of generality, we assume that \cref{alg:gre} is implemented as follows. When arm $1$ is pulled for the $s$-th time, the algorithm generates an infinite i.i.d.\ sequence $(\hat{\mu}^{(s)}_\ell)_\ell \sim p(\cH_{1, s})$. Then, instead of sampling $\hat{\mu}_{1, t} \sim p(\cH_{1, s})$ in round $t$ when $T_{1, t - 1} = s$, $\hat{\mu}_{1, t}$ is substituted with $\hat{\mu}^{(s)}_t$. Let $M = \set{t \in [n]: \max_{j > 1} \hat{\mu}_{j, t} \leq \tau_i}$ be round indices where the values of all suboptimal arms are at most $\tau_i$ and
\begin{align*}
  A_s
  = \set{t \in M: \hat{\mu}_t^{(s)} \leq \tau_i, \, T_{1, t - 1} = s}
\end{align*}
be its subset where the value of arm $1$ is at most $\tau_i$ and the arm was pulled $s$ times before. Then
\begin{align*}
  \sum_{t = 1}^n \I{I_t = i, \, E_{i, t} \text{ occurs}}
  \leq \sum_{t = 1}^n \I{\max_j \hat{\mu}_{j, t} \leq \tau_i}
  = \sum_{s = 0}^{n - 1} \underbrace{\sum_{t = 1}^n
  \I{\max_j \hat{\mu}_{j, t} \leq \tau_i, \, T_{1, t - 1} = s}}_{|A_s|}\,.
\end{align*}
In the next step, we bound $|A_s|$. Let
\begin{align*}
  \Lambda_s
  = \min \set{t \in M: \hat{\mu}_t^{(s)} > \tau_i, \, T_{1, t - 1} \geq s}
\end{align*}
be the index of the first round in $M$ where the value of arm $1$ is larger than $\tau_i$ and the arm was pulled at least $s$ times before. If such $\Lambda_s$ does not exist, we set $\Lambda_s = n$. Let
\begin{align*}
  B_s
  = \set{t \in M \cap [\Lambda_s]: \hat{\mu}_t^{(s)} \leq \tau_i, \, T_{1, t - 1} \geq s}
\end{align*}
be a subset of $M \cap [\Lambda_s]$ where the value of arm $1$ is at most $\tau_i$ and the arm was pulled at least $s$ times before.

We claim that $A_s \subseteq B_s$. By contradiction, suppose that there exists $t \in A_s$ such that $t \not\in B_s$. Then it must be true that $\Lambda_s < t$, from the definitions of $A_s$ and $B_s$. From the definition of $\Lambda_s$, we know that arm $1$ was pulled in round $\Lambda_s$, after it was pulled at least $s$ times before. Therefore, it cannot be true that $T_{1, t - 1} = s$, and thus $t \not\in A_s$. Therefore, $A_s \subseteq B_s$ and $|A_s| \leq |B_s|$. In the next step, we bound $|B_s|$ in expectation.

Let $\cF_t = \sigma(\cH_{1, T_{1, t}}, \dots, \cH_{K, T_{K, t}}, I_1, \dots, I_t)$ be the $\sigma$-algebra generated by arm histories and pulled arms by the end of round $t$, for $t \in [n] \cup \set{0}$. Let $P_s = \min \set{t \in [n]: T_{1, t - 1} = s}$ be the index of the first round where arm $1$ was pulled $s$ times before. If such $P_s$ does not exist, we set $P_s = n + 1$. Note that $P_s$ is a stopping time with respect to filtration $(\cF_t)_t$. Hence, $\cG_s = \cF_{P_s - 1}$ is well-defined and thanks to $|A_s| \leq n$, we have
\begin{align*}
  \E{|A_s|}
  = \E{\min \set{\condE{|A_s|}{\cG_s}, n}}
  \leq \E{\min \set{\condE{|B_s|}{\cG_s}, n}}\,.
\end{align*}
We claim that $\condE{|B_s|}{\cG_s} \leq 1 / Q_{1, s}(\tau_i) - 1$. First, note that $|B_s|$ can be rewritten as
\begin{align*}
  |B_s|
  = \sum_{t = P_s}^{\Lambda_s} \epsilon_t \rho_t\,,    
\end{align*}
where $\epsilon_t = \I{\max_{j > 1} \hat{\mu}_{j, t} \leq \tau_i}$ control which $\rho_t = \I{\hat{\mu}^{(s)}_t \leq \tau_i}$ contribute to the sum. Now recall Theorem 5.2 from Chapter III of \citet{doob53stochastic}.

\begin{theorem}
Let $X_1, X_2, \dots$ and $Z_1, Z_2, \dots$ be two sequences of random variables and $(\cF_t)_t$ be a filtration. Let $(X_t)_t$ be i.i.d., $X_t$ be $\cF_t$ measurable, $Z_t \in \set{0, 1}$, and $Z_t$ be $\cF_{t - 1}$ measurable. Let $N_t = \min \set{t > N_{t - 1}: Z_t = 1}$ for $t \in [m]$, $N_0 = 0$, and assume that $N_m < \infty$ almost surely. Let $X_t' = X_{N_t}$ for $t \in [m]$. Then $(X_t')_{t = 1}^m$ is i.i.d.\ and its elements have the same distribution as $X_1$.
\end{theorem}

By the above theorem and the definition of $\Lambda_s$, $|B_s|$ has the same distribution as the number of failed independent draws from $\mathrm{Ber}(Q_{1, s}(\tau_i))$ until the first success, capped at $n - P_s$. It is well known that the expected value of this quantity, without the cap, is bounded by $1 / Q_{1, s}(\tau_i) - 1$.

Finally, we chain all inequalities and get
\begin{align*}
  \E{\sum_{t = 1}^n \I{I_t = i, E_{i, t} \text{ occurs}}}
  \leq \sum_{s = 0}^{n - 1} \E{\min \set{1 / Q_{1, s}(\tau_i) - 1, n}}\,.
\end{align*}
This concludes our proof.


\section{Proof of \cref{thm:giro upper bound}}
\label{sec:giro proof}

This proof has two parts.

\subsubsection*{Upper Bound on $b_i$ in \cref{thm:gre upper bound} (\cref{sec:gre analysis})}

Fix suboptimal arm $i$. To simplify notation, we abbreviate $Q_{i, s}(\tau_i)$ as $Q_{i, s}$. Our first objective is to bound
\begin{align*}
  b_i
  = \sum_{s = 0}^{n - 1} \prob{Q_{i, s} > 1 / n} + 1\,.
\end{align*}
Fix the number of pulls $s$. When the number of pulls is \say{small}, $\displaystyle s \leq \frac{8 \alpha}{\Delta_i^2} \log n$, we bound $\prob{Q_{i, s} > 1 / n}$ trivially by $1$. When the number of pulls is \say{large}, $\displaystyle s > \frac{8 \alpha}{\Delta_i^2} \log n$, we divide the proof based on the event that $V_{i, s}$ is not much larger than its expectation. Define
\begin{align*}
  E
  = \set{V_{i, s} - (\mu_i + a) s \leq \frac{\Delta_i s}{4}}\,.
\end{align*}
On event $E$,
\begin{align*}
  Q_{i, s}
  = \condprob{U_{i, s} - (\mu_i + a) s \geq \frac{\Delta_i s}{2}}{V_{i, s}} 
  \leq \condprob{U_{i, s} - V_{i, s} \geq \frac{\Delta_i s}{4}}{V_{i, s}} 
  \leq \exp\left[- \frac{\Delta_i^2 s}{8 \alpha}\right] 
  \leq n^{-1}\,,
\end{align*}
where the first inequality is from the definition of event $E$, the second inequality is by Hoeffding's inequality, and the third inequality is by our assumption on $s$. On the other hand, event $\bar{E}$ is unlikely because
\begin{align*}
  \prob{\bar{E}}
  \leq \exp\left[- \frac{\Delta_i^2 s}{8 \alpha}\right]
  \leq n^{-1}\,,
\end{align*}
where the first inequality is by Hoeffding's inequality and the last inequality is by our assumption on $s$. Now we apply the last two inequalities to
\begin{align*}
  \prob{Q_{i, s} > 1 / n}
  & = \E{\condprob{Q_{i, s} > 1 / n}{V_{i, s}} \I{E}} +
  \E{\condprob{Q_{i, s} > 1 / n}{V_{i, s}} \I{\bar{E}}} \\
  & \leq 0 + \E{\I{\bar{E}}}
  \leq n^{-1}\,.
\end{align*}
Finally, we chain our upper bounds for all $s \in [n]$ and get the upper bound on $b_i$ in \eqref{eq:upper bound}.

\subsubsection*{Upper Bound on $a_i$ in \cref{thm:gre upper bound} (\cref{sec:gre analysis})}

Fix suboptimal arm $i$. Our second objective is to bound
\begin{align*}
  a_i
  = \sum_{s = 0}^{n - 1} \E{\min \set{\frac{1}{Q_{1, s}(\tau_i)} - 1, n}}\,.
\end{align*}
We redefine $\tau_i$ as $\tau_i = (\mu_1 + a) / \alpha - \Delta_i / (2 \alpha)$ and abbreviate $Q_{1, s}(\tau_i)$ as $Q_{1, s}$. Since $i$ is fixed, this slight abuse of notation should not cause any confusion. For $s > 0$, we have
\begin{align*}
  Q_{1, s}
  = \condprob{\frac{U_{1, s}}{\alpha s} \geq
  \frac{\mu_1 + a}{\alpha} - \frac{\Delta_i}{2 \alpha}}{V_{1, s}}\,.
\end{align*}
Let $F_s = 1 / Q_{1, s} - 1$. Fix the number of pulls $s$. When $s = 0$, $Q_{1, s} = 1$ and $\E{\min \set{F_s, n}} = 0$. When the number of pulls is \say{small}, $\displaystyle 0 < s \leq \frac{16 \alpha}{\Delta_i^2} \log n$, we apply the upper bound from \cref{thm:optimism upper bound} in \cref{sec:optimism upper bound} and get
\begin{align*}
  \E{\min \set{F_s, n}}
  \leq \E{1 / Q_{1, s}}
  \leq \E{1 / \condprob{U_{1, s} \geq (\mu_1 + a) s}{V_{1, s}}}
  \leq c\,,
\end{align*}
where $c$ is defined in \cref{thm:giro upper bound}. The last inequality is by \cref{thm:optimism upper bound} in \cref{sec:optimism upper bound}.

When the number of pulls is \say{large}, $\displaystyle s > \frac{16 \alpha}{\Delta_i^2} \log n$, we divide the proof based on the event that $V_{1, s}$ is not much smaller than its expectation. Define
\begin{align*}
  E
  = \set{(\mu_1 + a) s - V_{1, s} \leq \frac{\Delta_i s}{4}}\,.
\end{align*}
On event $E$,
\begin{align*}
  Q_{1, s}
  & = \condprob{(\mu_1 + a) s - U_{1, s} \leq \frac{\Delta_i s}{2}}{V_{1, s}}
  = 1 - \condprob{(\mu_1 + a) s - U_{1, s} > \frac{\Delta_i s}{2}}{V_{1, s}} \\
  & \geq 1 - \condprob{V_{1, s} - U_{1, s} > \frac{\Delta_i s}{4}}{V_{1, s}}
  \geq 1 - \exp\left[- \frac{\Delta_i^2 s}{8 \alpha}\right] 
  \geq \frac{n^2 - 1}{n^2}\,,
\end{align*}
where the first inequality is from the definition of event $E$, the second inequality is by Hoeffding's inequality, and the third inequality is by our assumption on $s$. The above lower bound yields
\begin{align*}
  F_s
  = \frac{1}{Q_{1, s}} - 1
  \leq \frac{n^2}{n^2 - 1} - 1
  = \frac{1}{n^2 - 1}
  \leq n^{-1}
\end{align*}
for $n \geq 2$. On the other hand, event $\bar{E}$ is unlikely because
\begin{align*}
  \prob{\bar{E}}
  \leq \exp\left[- \frac{\Delta_i^2 s}{8 \alpha}\right]
  \leq n^{-2}\,,
\end{align*}
where the first inequality is by Hoeffding's inequality and the last inequality is by our assumption on $s$. Now we apply the last two inequalities to
\begin{align*}
  \E{\min \set{F_s, n}}
  & = \E{\condE{\min \set{F_s, n}}{V_{1, s}} \I{E}} +
  \E{\condE{\min \set{F_s, n}}{V_{1, s}} \I{\bar{E}}} \\
  & \leq \E{n^{-1} \I{E}} + \E{n \I{\bar{E}}}
  \leq 2 n^{-1}\,.
\end{align*}
Finally, we chain our upper bounds for all $s \in [n]$ and get the upper bound on $a_i$ in \eqref{eq:upper bound}. This concludes our proof.


\section{Upper Bound on the Expected Inverse Probability of Being Optimistic}
\label{sec:optimism upper bound}

\cref{thm:optimism upper bound} provides an upper bound on the expected inverse probability of being optimistic,
\begin{align*}
  \E{1 / \condprob{U_{1, s} \geq (\mu_1 + a) s}{V_{1, s}}}\,,
\end{align*}
which is used in \cref{sec:giro analysis,sec:giro proof}. In the bound and its analysis, $n$ is $s$, $p$ is $\mu_1$, $x$ is $V_{1, s} - a s$, and $y$ is $U_{1, s}$.

\begin{theorem}
\label{thm:optimism upper bound} Let $m = (2 a + 1) n$ and $\displaystyle b = \frac{2 a + 1}{a (a + 1)} < 2$. Then
\begin{align*}
  W =
  \sum_{x = 0}^n B(x; n, p)
  \left[\sum_{y = \ceils{(a + p) n}}^m
  B\left(y; m, \frac{a n + x}{m}\right)\right]^{-1} \leq
  \frac{2 e^2 \sqrt{2 a + 1}}{\sqrt{2 \pi}} \exp\left[\frac{8 b}{2 - b}\right]
  \left(1 + \sqrt{\frac{2 \pi}{4 - 2 b}}\right)\,.
\end{align*}
\end{theorem}
\begin{proof}
First, we apply the upper bound from \cref{lem:history upper bound} for
\begin{align*}
  f(x) =
  \left[\sum_{y = \ceils{(a + p) n}}^m
  B\left(y; m, \frac{a n + x}{m}\right)\right]^{-1}\,.
\end{align*}
Note that this function decreases in $x$, as required by \cref{lem:history upper bound}, because the probability of observing at least $\ceils{(a + p) n}$ ones increases with $x$, for any fixed $\ceils{(a + p) n}$. The resulting upper bound is
\begin{align*}
  W \leq
  \sum_{i = 0}^{i_0 - 1} \exp[-2 i^2]
  \left[\sum_{y = \ceils{(a + p) n}}^m
  B\left(y; m, \frac{(a + p) n - (i + 1) \sqrt{n}}{m}\right)\right]^{-1} +
  \exp[-2 i_0^2]
  \left[\sum_{y = \ceils{(a + p) n}}^m
  B\left(y; m, \frac{a n}{m}\right)\right]^{-1}\,,
\end{align*}
where $i_0$ is the smallest integer such that $(i_0 + 1) \sqrt{n} \geq p n$, as defined in \cref{lem:history upper bound}.

Second, we bound both above reciprocals using \cref{lem:bootstrap lower bound}. The first term is bounded for $x = p n - (i + 1) \sqrt{n}$ as
\begin{align*}
  \left[\sum_{y = \ceils{(a + p) n}}^m
  B\left(y; m, \frac{(a + p) n - (i + 1) \sqrt{n}}{m}\right)\right]^{-1} \leq
  \frac{e^2 \sqrt{2 a + 1}}{\sqrt{2 \pi}}
  \exp[b (i + 2)^2]\,.
\end{align*}
The second term is bounded for $x = 0$ as
\begin{align*}
  \left[\sum_{y = \ceils{(a + p) n}}^m
  B\left(y; m, \frac{a n}{m}\right)\right]^{-1} \leq
  \frac{e^2 \sqrt{2 a + 1}}{\sqrt{2 \pi}}
  \exp\left[b \frac{(p n + \sqrt{n})^2}{n}\right] \leq
  \frac{e^2 \sqrt{2 a + 1}}{\sqrt{2 \pi}}
  \exp[b (i_0 + 2)^2]\,,
\end{align*}
where the last inequality is from the definition of $i_0$. Then we chain the above three inequalities and get
\begin{align*}
  W \leq
  \frac{e^2 \sqrt{2 a + 1}}{\sqrt{2 \pi}} \sum_{i = 0}^{i_0}
  \exp[-2 i^2 + b (i + 2)^2]\,.
\end{align*}
Now note that
\begin{align*}
  2 i^2 - b (i + 2)^2 =
  (2 - b) \left(i^2 - \frac{4 b i}{2 - b} + \frac{4 b^2}{(2 - b)^2}\right) -
  \frac{4 b^2}{2 - b} - 4b =
  (2 - b) \left(i - \frac{2 b}{2 - b}\right)^2 - \frac{8 b}{2 - b}\,.
\end{align*}
It follows that
\begin{align*}
  W
  & \leq \frac{e^2 \sqrt{2 a + 1}}{\sqrt{2 \pi}} \sum_{i = 0}^{i_0}
  \exp\left[- (2 - b) \left(i - \frac{2 b}{2 - b}\right)^2 + \frac{8 b}{2 - b}\right] \\
  & \leq \frac{2 e^2 \sqrt{2 a + 1}}{\sqrt{2 \pi}} \exp\left[\frac{8 b}{2 - b}\right]
  \sum_{i = 0}^\infty \exp\left[- (2 - b) i^2\right] \\
  & \leq \frac{2 e^2 \sqrt{2 a + 1}}{\sqrt{2 \pi}} \exp\left[\frac{8 b}{2 - b}\right]
  \left(1 + \int_{u = 0}^\infty
  \exp\left[- \frac{u^2}{\frac{2}{4 - 2 b}}\right] \dif u\right) \\
  & \leq \frac{2 e^2 \sqrt{2 a + 1}}{\sqrt{2 \pi}} \exp\left[\frac{8 b}{2 - b}\right]
  \left(1 + \sqrt{\frac{2 \pi}{4 - 2 b}}\right)\,.
\end{align*}
This concludes our proof.
\end{proof}

\begin{lemma}
\label{lem:history upper bound} Let $f(x) \geq 0$ be a decreasing function of $x$ and $i_0$ be the smallest integer such that $(i_0 + 1) \sqrt{n} \geq p n$. Then
\begin{align*}
  \sum_{x = 0}^n B(x; n, p) f(x) \leq
  \sum_{i = 0}^{i_0 - 1} \exp[-2 i^2] f(p n - (i + 1) \sqrt{n}) + \exp[-2 i_0^2] f(0)\,.
\end{align*}
\end{lemma}

\begin{proof}
Let
\begin{align*}
  \mathcal{X}_i =
  \begin{cases}
    (\max \set{p n - \sqrt{n}, \, 0}, \, n]\,, & i = 0\,; \\
    (\max \set{p n - (i + 1) \sqrt{n}, \, 0}, \, p n - i \sqrt{n}]\,, & i > 0\,;
  \end{cases}
\end{align*}
for $i \in [i_0] \cup \set{0}$. Then $\set{\mathcal{X}_i}_{i = 0}^{i_0}$ is a partition of $[0, n]$. Based on this observation,
\begin{align*}
  \sum_{x = 0}^n B(x; n, p) f(x)
  & = \sum_{i = 0}^{i_0} \sum_{x = 0}^n
  \I{x \in \mathcal{X}_i} B(x; n, p) f(x) \\
  & \leq \sum_{i = 0}^{i_0 - 1} f(p n - (i + 1) \sqrt{n})
  \sum_{x = 0}^n \I{x \in \mathcal{X}_i} B(x; n, p) +
  f(0) \sum_{x = 0}^n \I{x \in \mathcal{X}_{i_0}} B(x; n, p) \,,
\end{align*}
where the inequality holds because $f(x)$ is a decreasing function of $x$. Now fix $i > 0$. Then from the definition of $\mathcal{X}_i$ and Hoeffding's inequality,
\begin{align*}
  \sum_{x = 0}^n \I{x \in \mathcal{X}_i} B(x; n, p) \leq
  \condprob{X \leq p n - i \sqrt{n}}{X \sim B(n, p)} \leq
  \exp[-2 i^2]\,.
\end{align*}
Trivially, $\sum_{x = 0}^n \I{x \in \mathcal{X}_0} B(x; n, p) \leq 1 = \exp[- 2 \cdot 0^2]$. Finally, we chain all inequalities and get our claim.
\end{proof}

\begin{lemma}
\label{lem:bootstrap lower bound} Let $x \in [0, p n]$, $m = (2 a + 1) n$, and $\displaystyle b = \frac{2 a + 1}{a (a + 1)}$. Then for any integer $n > 0$,
\begin{align*}
  \sum_{y = \ceils{(a + p) n}}^m B\left(y; m, \frac{a n + x}{m}\right) \geq
  \frac{\sqrt{2 \pi}}{e^2 \sqrt{2 a + 1}}
  \exp\left[- b \frac{(p n + \sqrt{n} - x)^2}{n}\right]\,.
\end{align*}
\end{lemma}
\begin{proof}
By \cref{lem:stirling},
\begin{align*}
  B\left(y; m, \frac{a n + x}{m}\right) \geq
  \frac{\sqrt{2 \pi}}{e^2} \sqrt{\frac{m}{y (m - y)}}
  \exp\left[- \frac{(y - a n - x)^2}
  {m \frac{a n + x}{m} \frac{(a + 1) n - x}{m}}\right]\,.
\end{align*}
Now note that
\begin{align*}
  \frac{y (m - y)}{m} \leq
  \frac{1}{m} \frac{m^2}{4} =
  \frac{(2 a + 1) n}{4}\,.
\end{align*}
Moreover, since $x \in [0, p n]$,
\begin{align*}
  m \frac{a n + x}{m} \frac{(a + 1) n - x}{m} \geq
  m \frac{a n}{m} \frac{(a + 1) n}{m} =
  \frac{a (a + 1) n}{2 a + 1} =
  \frac{n}{b}\,,
\end{align*}
where $b$ is defined in the claim of this lemma. Now we combine the above three inequalities and have
\begin{align*}
  B\left(y; m, \frac{a n + x}{m}\right) \geq
  \frac{2 \sqrt{2 \pi}}{e^2 \sqrt{2 a + 1}} \frac{1}{\sqrt{n}}
  \exp\left[- b \frac{(y - a n - x)^2}{n}\right]\,,
\end{align*}
Finally, note the following two facts. First, the above lower bound decreases in $y$ when $y \geq (a + p) n$ and $x \leq p n$. Second, by the pigeonhole principle, there exist at least $\floors{\sqrt{n}}$ integers between $(a + p) n$ and $(a + p) n + \sqrt{n}$, starting with $\ceils{(a + p) n}$. These observations lead to a trivial lower bound
\begin{align*}
  \sum_{y = \ceils{(a + p) n}}^m B\left(y; m, \frac{a n + x}{m}\right)
  & \geq \frac{\floors{\sqrt{n}}}{\sqrt{n}} \frac{2 \sqrt{2 \pi}}{e^2 \sqrt{2 a + 1}}
  \exp\left[- b \frac{(p n + \sqrt{n} - x)^2}{n}\right] \\
  & \geq \frac{\sqrt{2 \pi}}{e^2 \sqrt{2 a + 1}}
  \exp\left[- b \frac{(p n + \sqrt{n} - x)^2}{n}\right]\,.
\end{align*}
The last inequality is from $\floors{\sqrt{n}} / \sqrt{n} \geq 1 / 2$, which holds for $n \geq 1$. This concludes our proof.
\end{proof}

\begin{lemma}
\label{lem:stirling} For any binomial probability,
\begin{align*}
  B(x; n, p) \geq
  \frac{\sqrt{2 \pi}}{e^2} \sqrt{\frac{n}{x (n - x)}}
  \exp\left[- \frac{(x - p n)^2}{p (1 - p) n}\right]\,.
\end{align*}
\end{lemma}
\begin{proof}
By Stirling's approximation, for any integer $k \geq 0$,
\begin{align*}
  \sqrt{2 \pi} k^{k + \frac{1}{2}} e^{- k} \leq
  k! \leq
  e k^{k + \frac{1}{2}} e^{- k}\,.
\end{align*}
Therefore, any binomial probability can be bounded from below as
\begin{align*}
  B(x; n, p) =
  \frac{n!}{x! (n - x)!} p^x q^{n - x} \geq
  \frac{\sqrt{2 \pi}}{e^2} \sqrt{\frac{n}{x (n - x)}}
  \left(\frac{p n}{x}\right)^x \left(\frac{q n}{n - x}\right)^{n - x}\,,
\end{align*}
where $q = 1 - p$. Let
\begin{align*}
  d(p_1, p_2) =
  p_1 \log \frac{p_1}{p_2} + (1 - p_1) \log \frac{1 - p_1}{1 - p_2}
\end{align*}
be the KL divergence between Bernoulli random variables with means $p_1$ and $p_2$. Then
\begin{align*}
  \left(\frac{p n}{x}\right)^x \left(\frac{q n}{n - x}\right)^{n - x}
  & = \exp\left[x \log\left(\frac{p n}{x}\right) +
  (n - x) \log\left(\frac{q n}{n - x}\right)\right] \\
  & = \exp\left[- n \left(\frac{x}{n} \log\left(\frac{x}{p n}\right) +
  \frac{n - x}{n} \log\left(\frac{n - x}{q n}\right)\right)\right] \\
  & = \exp\left[- n d\left(\frac{x}{n}, p\right)\right] \\
  & \geq \exp\left[- \frac{(x - p n)^2}{p (1 - p) n}\right]\,,
\end{align*}
where the inequality is from $\displaystyle d(p_1, p_2) \leq \frac{(p_1 - p_2)^2}{p_2 (1 - p_2)}$. Finally, we chain all inequalities and get our claim.
\end{proof}

\end{document}